\documentclass[12pt,table,reqno]{amsart}

\usepackage{amsfonts,amsmath,amssymb,amsthm}
\usepackage{parskip}
\usepackage{siunitx}
\usepackage{etoolbox}
\usepackage{enumitem} 
\usepackage[numbers,sort]{natbib}
\usepackage{hyperref}
\usepackage{mathrsfs}

\usepackage[foot]{amsaddr}

\newtheorem{thm}{Theorem}
\newtheorem{theorem}[thm]{Theorem}

\newtheorem{remark}[thm]{Remark}
\newtheorem{proposition}[thm]{Proposition}

\newcommand{\energy}{\phi}
\def\<{\begin{equation}}
\def\>{\end{equation}}

\DeclareMathOperator*{\expectation}{\mathbb{E}}

\begin{document}

\title[on approximating $\nabla \MakeLowercase{f}$ with neural networks]{on approximating $\nabla f$\\with neural networks }

\author{saeed saremi}
\address{\rm NNAISENSE, UC BERKELEY}
\email{saeed@nnaisense.com}

\date{November 6, 2019}

\begin{abstract}
Consider a feedforward neural network $\psi: \mathbb{R}^d\rightarrow \mathbb{R}^d$ such that $\psi\approx \nabla f$, where $f:\mathbb{R}^d \rightarrow \mathbb{R}$ is a smooth function, therefore $\psi$ must satisfy $\partial_j \psi_i = \partial_i \psi_j$ pointwise. We prove a theorem that a $\psi$ network with more than one hidden layer can only represent one feature in its first hidden layer; this is a dramatic departure from the well-known results for one hidden layer. The proof of the theorem is straightforward, where two backward paths and a weight-tying matrix play the key roles.  We then present the alternative, the implicit parametrization, where the neural network is $\phi: \mathbb{R}^d \rightarrow \mathbb{R}$  and $\nabla \phi \approx \nabla f$; in addition, a ``soft analysis'' of $\nabla \phi$ gives a dual perspective on the theorem.  Throughout, we come back to recent  probabilistic models that are formulated as $\nabla \phi \approx \nabla f$, and conclude with a critique of denoising autoencoders.
\end{abstract}
\maketitle
There is a paragraph in~\citep{saremi2018deep}, titled ``failures of estimating the score function directly''\textemdash they experimented and reported the failures of training a neural network $\psi:\mathbb{R}^d \rightarrow \mathbb{R}^d$ to approximate  $\nabla f$, where $f:\mathbb{R}^d \rightarrow \mathbb{R}$ was the log probability density function, $\nabla f$ the score function, and the learning objective was the denoising score matching~\citep{hyvarinen2005estimation,vincent2011connection}. The problem is in fact general, and the question is, if a (deep) neural network has the capacity to approximate $\nabla f$ for a general smooth $f$. We expect the capacity of $\psi$ to be limited, because the neural network must satisfy
$$ \partial_j \psi_i(x,\theta) = \partial_i \psi_j(x,\theta)$$
pointwise in $\mathbb{R}^d$, and these equations put constraints on the parameters $\theta \in \mathbb{R}^p$. However, what these constraints  are and how limiting they are have never been explored before beyond a single hidden layer~\citep{vincent2011connection}. Note that $\partial_j \psi_i=\partial_i \psi_j$ is {given}, a direct consequence of $f$ being a smooth function, and the fact that $\psi \approx \nabla f$. 

{We prove that to satisfy $\partial_j \psi_i=\partial_i \psi_j$ for a neural network with more than one hidden layer (i.e. with depth $L>2$), the input weights $\theta_m^{(0)} \in \mathbb{R}^d$ and the output weights $\theta_n^{(L-1)} \in \mathbb{R}^d$ must be all parallel to each other. Therefore, in \emph{explicitly} parameterizing  $$\nabla f:\mathbb{R}^d \rightarrow \mathbb{R}^d,$$ the neural network is required to represent only one feature in its first hidden layer. This surprising result is encapsulated in Theorem~\ref{thm:onelegged}, which we build up to starting from one hidden layer. We then present the alternative; it is the \emph{implicit} parametrization,  where the neural network is $$\phi: \mathbb{R}^d \rightarrow \mathbb{R},$$ and $\nabla \phi \approx \nabla f$. 

We discuss this implicit parametrization solution in some detail with an example from~\citep{saremi2019neural}, where we outline how to set up the learning problem and how to proceed with the optimization. We then give a different perspective on Theorem~\ref{thm:onelegged} with a ``soft analysis'' of $\nabla \phi$, where we show that $\nabla \phi$ does not map to a (conventional) feedforward neural network if $\phi$ has more than one hidden layer. We conclude the paper with a critique of denoising autoencoders, written as $\psi: \mathbb{R}^d \rightarrow \mathbb{R}^d$, in the limit where they approximate the score function~\citep{alain2014regularized}.}

{\bf One hidden layer.} \label{sec:score1} Consider $\psi$ with one hidden layer and a linear readout:
\begin{equation} \label{eq:psi1} \psi_i(x) = \sum \theta^{(1)}_{im} \sigma(\langle \theta^{(0)}_m,x\rangle), \end{equation}
where  $\sigma: \mathbb{R} \rightarrow \mathbb{R}$ is the activation function, $\theta^{(1)}_{im}\in\mathbb{R},~\theta^{(0)}_{m} \in \mathbb{R}^d$, $m$ is the index for neurons in the hidden layer, and $\sum$ denotes the sum over the {repeated} indices. The biases are dropped throughout the paper for the ease of notation; they do not affect the results as we are only concerned with the derivatives with respect to the inputs, not the parameters. It follows,
$$ \partial_j \psi_i(x) =  \sum \theta^{(1)}_{im} \theta^{(0)}_{mj}~\sigma'(\langle \theta^{(0)}_m,x\rangle). $$
To satisfy the constraints $\partial_j \psi_i = \partial_i \psi_j$ pointwise, the following must hold indexwise, \begin{equation} \label{eq:constr1} \theta^{(1)}_{im} \theta^{(0)}_{mj}=\theta^{(1)}_{jm} \theta^{(0)}_{mi}.\end{equation}
\begin{remark} \label{remark:relu} \rm
	Strictly speaking, the indexwise condition is not a necessary condition if $\sigma'$ does not take an infinite number of values in $\mathbb{R}$. The ReLU is an example of such a function. But for any such ``weird" functions, there exists a sequence of smooth functions as limits. An example is	$$ \sigma_\beta(x) = x/(1+e^{-\beta x}),$$
	where ReLU is the  uniform limit as $\beta \rightarrow \infty$. In this particular example, the condition in Eq.~\ref{eq:constr1} is in fact a  necessary condition for any $\beta < \infty$ (see the remark below). In this work, we make the mild assumption that either $\sigma'$ takes an infinite number of distinct values in $\mathbb{R}$ or there exists a sequence such that it converges to $\sigma'$ and every function in the sequence does. 
	\end{remark}	
	
\begin{remark} \label{remark:lebesgue} \rm
In addition, in stating the necessary condition, we followed the ``Lebesgue philosophy''~\citep{tao2010epsilon} and ignored the set of measure zero where $\partial_j\psi_i=\partial_i\psi_j$ can hold. This set is constructed by $\theta^{(0)}_m=\theta^{(0)}_{m'}$ (for all $m$ and $m'$) together with $$\sum (\theta^{(1)}_{jm} \theta^{(0)}_{mi}-\theta^{(1)}_{jm} \theta^{(0)}_{mi})=0.$$
\end{remark}

One can visualize $\theta^{(1)}_{im} \theta^{(0)}_{mj}$ as the multiplication of weights on the backward path $(i,m,j)$ taken from unit $i$ in the output layer to the unit $j$ in the input layer passing through $m$ in the hidden layer. Eq.~\ref{eq:constr1} is then seen as such: the weights multiplied on the backward paths $(i,m,j)$ and $(j,m,i)$ must match. The constraints are satisfied with a one-dimensional array $S$ with the size of the hidden layer: 
\<
\theta^{(1)}_{m} = S_m \theta^{(0)}_{m},
\>
where $S_m\in \mathbb{R}$, $\theta^{(1)}_{m} \in \mathbb{R}^d$ is the readout weights from unit $m$, and $\theta^{(0)}_{m} \in \mathbb{R}^d$ is the input weight to the same unit.  In other words, for each neuron $m$ in the hidden layer, the readout weight $\theta^{(1)}_{m}$ is  tied to the input weight  $\theta^{(0)}_{m}$ with the scaling factor $S_m$.  The $S$ itself is  free to change, and crucially $\theta^{(0)}_{m}$ themselves are also free (the latter will change dramatically for two and more hidden layers).  
\begin{remark} \rm
	The above tied weights phenomenon for the one hidden layer $\psi$ network seems to be ``universal'' and well known in the literature, where they ``pop up'' in a lot of places, from restricted Boltzmann machines to denoising autoencoders. See~\citep{vincent2011connection} for another derivation and further discussion. 
\end{remark}


{\bf Two hidden layers.} \label{sec:score2} Next, we present our key contribution on the low expressivity of  $\psi$ networks with two hidden layers, where a qualitative change emerges compared to the one hidden layer. The proof in Proposition~\ref{prop:twohidden} is the key to proving Theorem~\ref{thm:onelegged}. 

\begin{proposition}[A $\psi$ network with two hidden layers can only represent one feature in its first hidden layer] \label{prop:twohidden} Consider $\psi$  with two hidden layers: \begin{equation} \label{eq:psi2} \psi_i(x) = \sum \theta^{(2)}_{in}\sigma(\theta^{(1)}_{nm} \sigma(\langle \theta^{(0)}_m,x\rangle)), \end{equation}
where $\theta^{(2)}_{in},\theta^{(1)}_{nm}\in\mathbb{R}, \theta^{(0)}_{m} \in \mathbb{R}^d$, and $\sum$ denotes the sum over the repeated indices. Then the constrains $\partial_j\psi_i=\partial_i\psi_j$ are satisfied only if (the trivial case, $\theta_n^{(2)}=0$ for all $n$, aside) the input weights $\theta^{(0)}_{m} \in \mathbb{R}^d$ and the output weights $\theta^{(2)}_{n} \in \mathbb{R}^d$ are all parallel to each other.	
	\begin{proof}[\bf Proof]
We first compute $\partial_j \psi_i(x)$. It follows,
\< \label{eq:backprop} \partial_j \psi_i(x) =  \sum \theta^{(2)}_{in}\theta^{(1)}_{nm} \theta^{(0)}_{mj} \sigma'(\theta^{(1)}_{nm} \sigma(\langle \theta^{(0)}_m,x\rangle)) \sigma'(\langle \theta^{(0)}_m,x\rangle).\>
To satisfy $\partial_j \psi_i = \partial_i \psi_j$ pointwise for all $x$ in $\mathbb{R}^d$, the parameters $\theta$ must satisfy the following identity indexwise (see Remark~\ref{remark:relu} and~\ref{remark:lebesgue}):
\< \theta^{(2)}_{in}\theta^{(1)}_{nm} \theta^{(0)}_{mj} = \theta^{(2)}_{jn}\theta^{(1)}_{nm} \theta^{(0)}_{mi}.\>
Similar to the case of one hidden layer, we visualize $\theta^{(2)}_{in}\theta^{(1)}_{nm} \theta^{(0)}_{mj}$ as the multiplication of weights on the backward path taken from unit $i$ in the output layer to the unit $j$ in the input layer $(i,n,m,j)$. The weights multiplied on the two paths $(i,n,m,j)$ and $(j,n,m,i)$ must therefore match. Note that $(n,m)$ is common between the two paths, and the constraint is simplified to
\<\label{eq:inmj} \theta^{(2)}_{in} \theta^{(0)}_{mj} = \theta^{(2)}_{jn} \theta^{(0)}_{mi}.\>
The condition is therefore reduced to the one layer case with the important exception that $S$ is transformed to a weight-tying {\it matrix} connecting the last and first hidden layers:
\< \label{eq:key}
\theta^{(2)}_{n} = S_{nm} \theta^{(0)}_{m}.\\
\>
 The identity must hold indexwise.  The problem is that to satisfy the above constraints for all indices the input weights themselves become tied. This is easy to see by writing the indexwise constraint Eq.~\ref{eq:key} for $m$ and $m'\neq m$, which results in
\< S_{nm'} \theta^{(0)}_{m'}= S_{nm} \theta^{(0)}_{m}.\>
Input weights $\theta^{(0)}_{m}$ must therefore be all parallel to each other. In addition, from Eq.~\ref{eq:key}, the output weights $\theta^{(2)}_{n}$ must point in that same direction. This is the claim in the proposition. Obviously, there is an alternative  by setting $S_{nm}=0$ for all $n$, but that corresponds to $\theta_n^{(2)}=0$ for all $n$, which is the trivial case.
\end{proof}
\end{proposition}

{\bf More than two hidden layers.} Next, we prove the key theorem. 
\begin{proposition} \label{prop:easy} Consider a $\psi$ network with arbitrary depth $L$: $$ \psi_i(x) = \sum \theta^{(L-1)}_{in}\sigma(\theta^{(L-2)}_{n\beta}\sigma(\dots \sigma(\theta^{(1)}_{\alpha m} \sigma(\langle \theta^{(0)}_m,x\rangle))\cdots)),$$
where the first hidden layer is indexed by $m$, the last one by $n$, and the rest by $\alpha$, $\beta$, etc. Then to satisfy the constraints $\partial_j\psi_i =\partial_i\psi_j$ pointwise, the multiplication of the weights on the two backward paths $(i,n,\beta,\dots,\alpha,m,j)$ and $(j,n,\beta,\dots,\alpha,m,i)$ must be equal indexwise.
\end{proposition}
\begin{proof}[\bf Proof]
	We leave the proof to the reader; it comes from the chain rule and by inspecting the first step in the proof of Proposition~\ref{prop:twohidden}.
	\end{proof}

\begin{theorem}[A $\psi$ network with more than one hidden layer can only represent one feature in its first hidden layer] \label{thm:onelegged} \label{thorem:notexpressive} Consider $\psi$ with depth $L>2$:
\begin{equation} \nonumber \psi_i(x) = \sum \theta^{(L-1)}_{in}\sigma(\theta^{(L-2)}_{n\beta}\sigma(\dots \sigma(\theta^{(1)}_{\alpha m} \sigma(\langle \theta^{(0)}_m,x\rangle))\cdots)). \end{equation} Then the constrains $\partial_j\psi_i=\partial_i\psi_j$ are satisfied only if (the trivial case, $\theta_n^{(L-1)}=0$ for all $n$, aside) the input weights $\theta^{(0)}_{m} \in \mathbb{R}^d$ and the output weights $\theta^{(L-1)}_{n} \in \mathbb{R}^d$ are all parallel to each other.
 
\begin{proof}[\bf Proof]
	Start with the Proposition~\ref{prop:easy} and observe that the two backward paths $(i,n,\beta,\dots,\alpha,m,j)$ and $(j,n,\beta,\dots,\alpha,m,i)$ are the same except for the beginning and the end.  The necessary condition simplifies to the condition below that must hold indexwise:	$$ \theta^{(L-1)}_{in} \theta^{(0)}_{mj} = \theta^{(L-1)}_{jn} \theta^{(0)}_{mi}.$$
The rest follows (the steps after Eq.~\ref{eq:inmj}) from the proof in Proposition~\ref{prop:twohidden}.
\end{proof}	
\end{theorem}

{\bf The implicit parametrization of $\nabla f$.}  We now present the alternative, since we just learned that $\psi$ with more than one hidden layer is required to represent only one feature in its first hidden layer. The alternative is simple in fact: we parametrize $\phi: \mathbb{R}^d \rightarrow \mathbb{R}$ such that $\nabla \phi \approx \nabla f$. In this (implicit) parametrization, the learning objective $\mathcal{L}(\theta)=\mathcal{L}[\nabla \phi(\cdot,\theta)]$ is constructed as the function of the parameters $\theta$ of the $\phi$ network and minimized, where at its global minimum, $\nabla \phi \approx \nabla f$. However, the optimization is not straightforward. Let us consider an example to ground the ideas. The example is the most recent one, studied in~\citep{saremi2019neural}. In {\it neural empirical Bayes}, the learning objective $\mathcal{L}[\psi]$ as the functional of $\psi$ is given by
$$ \mathcal{L}[\psi]=\expectation \Vert X-Y-\sigma^2 \psi(Y) \Vert^2,$$
where $X$ is the ``clean'' random variable which takes values in $\mathbb{R}^d$, $Y$ is the smoothed $X$ that is obtained by corrupting samples from $X$ with the additive  Gaussian noise $N(0,\sigma^2 I_d)$, and the expectation $\mathbb{E}$ is over $X$ and $Y$. In this setup, at optimality, $\psi$ approximates $\nabla f$, where $f$ is the  log probability density function in the noisy space $Y$. We now switch to the $\phi$ network (it is {exactly} the ``deep energy estimator network''~\citep{saremi2018deep} with a negative sign). The learning objective expressed for the $\phi$ network as the function of the parameters $\theta$ is given by
$$ \mathcal{L}(\theta) = \expectation \Vert X-Y-\sigma^2\nabla \phi(Y,\theta) \Vert^2.$$
At optimality $\theta^*$, $\nabla \phi(\cdot,\theta^*) \approx \nabla f(\cdot)$.

This is an {implicit} parametrization of $\nabla f$, ``implicit'' because (in practice) the symbolic form of $\nabla \phi$ is not known~\citep{griewank2003mathematical} (next section is devoted to a ``soft analysis'' of this symbolic form for $L=3$). The optimization comes with a cost in the form of ``double backpropagation''~\citep{saremi2019neural} (a slight abuse of terminology since both $\nabla \phi$ and  $ \nabla \mathcal{L}$ are computed with  algorithmic/automatic differentiation, with techniques that are more general than \emph{the} backpropagation~\citep{rumelhart1986learning}, see~\citep{baydin2018automatic} for a survey and~\citep{griewank2012invented, schmidhuber2015deep} for a fascinating history of automatic differentiation and backpropagation). Essentially, $\phi$ must be differentiated first with respect to the inputs to compute the loss before each SGD update $$\theta \leftarrow \theta - \epsilon \nabla \mathcal{L}.$$ This appears computationally expensive, but the first differentiation is in the input space with $d\ll p$, and the stochastic gradient descent has been effectively used in optimizing $\mathcal{L}(\theta)$ in wide and deep {convolutional} neural network architectures presented in~\citep{saremi2019neural} and in {residual} network architectures~\citep{li2019annealed}.

{\bf $\nabla \phi \neq$ feedforward neural network.} Thereom~\ref{thm:onelegged} requires reflections and further analysis. This section is devoted to that. We start with an analysis of the symbolic form of $\nabla \phi$ for $L=2$, which corresponds to a feedforward neural network with tied weights; this maps to the results for the one hidden layer $\psi$ that we started this paper with. This calculation for $L=2$ is close to the analysis in~\citep{vincent2011connection} for denoising autoencoders with one hidden layer. We then study $\nabla \phi$ for $L=3$ and make observations about the fact that $\nabla \phi$ for $L>2$ does not map to a  (conventional) feedforward neural network, where a novel form of feedforward computation, dubbed as ``products of neural networks'', emerge. 

The analysis is as follows,
\begin{itemize}
	\item {\it One hidden layer.} Consider $\phi$ with one hidden layer,\begin{equation} \label{eq:phi1} \nonumber \energy(x) = \sum \theta^{(1)}_{m} \sigma(\langle \theta^{(0)}_m,x\rangle),\end{equation} 
where $\theta^{(1)}_m\in \mathbb{R},~\theta^{(0)}_{m} \in \mathbb{R}^d$. It follows,
\begin{equation}\label{eq:phi2psi}
	\psi_i(x) = \sum \theta^{(1)}_{m} \theta^{(0)}_{mi} \sigma'(\langle \theta^{(0)}_m,x\rangle).
\end{equation} 
This symbolic form corresponds to a feedforward neural network (a $\psi$ network) itself that is derived with the following mapping (compare Eq.~\ref{eq:psi1} with Eq.~\ref{eq:phi2psi}),
\begin{equation}\label{eq:map}
	\theta_{im}^{(1)} \leftarrow \theta^{(1)}_{m} \theta^{(0)}_{mi},~\theta_{m}^{(0)} \leftarrow \theta_{m}^{(0)},~\sigma \leftarrow \sigma'. 
\end{equation} 
This is consistent with our analysis of the $\psi$ network from the perspective of $\partial_j \psi_i = \partial_i \psi_j$ that we started this paper with: the output weights are tied to the input weights with the array $S_m = \theta_m^{(1)}$ that is given here in terms of the output weights $\theta_m^{(1)}$ of the $\phi$ network. In summary, $\psi$ with one hidden layer  can be written as the gradient of a corresponding $\phi$ with one hidden layer. We can also check independently that $\partial_j \psi_i=\partial_i \psi_j$ holds:$$ \partial_j \psi_i(x) = \sum \theta^{(1)}_{m} \theta^{(0)}_{mi} \theta^{(0)}_{mj} \sigma''(\langle \theta^{(0)}_m,x\rangle). $$
\item {\it Two (and more) hidden layers.} Now consider $\phi$ with two hidden layers,
$$\energy(x) = \sum \theta^{(2)}_{n}\sigma(\theta^{(1)}_{nm} \sigma(\langle \theta^{(0)}_m,x\rangle),$$
where $\theta^{(2)}_{n}, \theta^{(1)}_{nm} \in \mathbb{R},~\theta^{(0)}_m \in \mathbb{R}^d$. It follows, \begin{equation}\label{eq:score2layer}
  \psi_i(x) = \sum \theta^{(2)}_{n} \theta^{(1)}_{nm} \theta^{(0)}_{mi}\sigma'(\theta^{(1)}_{nm} \sigma(\langle \theta^{(0)}_m,x\rangle) \sigma'(\langle \theta^{(0)}_m,x\rangle).	
 \end{equation}
Note that $\langle \theta^{(0)}_m,x\rangle$ appears in two places and they are multiplied.   {\it Not a (conventional) feedforward computation!} It is easy to check that $\nabla \phi$ derived from a $\phi$ network with depth $L$ performs computations that correspond to  products of $L-1$ networks. 

Of course, this ``soft analysis'' of $\nabla \phi$ is not a proof that approximating $\nabla f$ cannot be achieved  by a ``conventional'' feedforward neural network of depth $L>2$ (Theorem~\ref{thm:onelegged} is such a proof), but it gives a different perspective on the theorem. We leave further analysis of this products of neural networks (and possible connections to {\it products of experts}~\citep{hinton1999products}) to future research.
\end{itemize}

{\bf Denoising autoencoders.} Next, we present a critique of denoising autoencoders  from the perspective of the results presented here.  For denoising autoencoders,$$\psi:\mathbb{R}^d \rightarrow \mathbb{R}^d$$ is the composition of an encoder and a decoder, and the neural network is optimized to reconstruct the clean samples $X$ from their Gaussian-corrupted versions $Y$. The learning objective is given by:
$$ \mathcal{L}[\psi] = \expectation \Vert X-\psi(Y)\Vert^2.$$
It was proven in~\citep{alain2014regularized} that at optimality, in the limit $\sigma\rightarrow 0$ (scaling/shifting factors aside), $\psi \approx \nabla f$ (see Theorem 1 in the reference). This was an important development as it generalized (at the cost of $\sigma\rightarrow 0$) the intriguing connection between score matching and denoising autoencoders that had been established in~\citep{vincent2011connection}. However, the question is:

\begin{quotation}
Does the  denoising autoencoder $\psi$ have the capacity to approximate $\nabla f$ for a general smooth function $f$?
\end{quotation} 

The general formulation of denoising autoencoders~\citep{alain2014regularized} did indeed generalize the connection between the denoising autoencoders and score matching~\citep{vincent2011connection}, but the problem is that a single hidden layer $\psi$ (with tied output weights) can be expressed as the gradient of an energy function (a $\phi$ network in our terminology), but as we showed here a deep(er) feedforward $\psi$ cannot be. There is a qualitative change (a ``phase transition'' of sort) that emerges in going from one hidden layer to two and more hidden layers. The concerns over $$\partial_j\psi_i=\partial_i\psi_j$$ and the impact of the constraints on the parameters of the denoising autoencoders was in fact mentioned in~\citep{alain2014regularized} in a section titled ``limited parametrization'', but exactly {\it how limited} those parameterizations are had remained an open problem. Theorem~\ref{thm:onelegged} is the surprising answer to that question.

{\bf Discussion.} We finish with a summary and the larger picture.\begin{itemize}
 \item[(i)]	The {\it depth} of neural networks is known to be a necessity for ``expressive power''~\citep{bengio2009learning, lu2017expressive}. Theorem~\ref{thm:onelegged} therefore makes a clear statement on the low expressivity of $\psi$ networks for the problem of approximating $\nabla f$, where we are left with two choices: (1) $\psi$ with one hidden layer and tied output weights (2) $\psi$ with only one feature represented in the first hidden layer.
 \item[(ii)] We presented $\phi$ as the implicit parametrization machinery and the only alternative for solving the problem of approximating $\nabla f$ with deep (expressive) neural networks. The framework presented here is general, but the recent examples we revisited here include ``deep energy estimator networks'', studied in the context of unnormalized probabilistic models~\citep{saremi2018deep,saremi2019neural}.
 \item[(iii)] A corollary of Theorem~\ref{thm:onelegged} is that the computations performed in $\nabla \phi$ with more than one hidden layer  cannot be approximated by a $\psi$ network in \emph{finite width.} This is because Theorem~\ref{thm:onelegged} states that $\psi$ with more than one hidden layer is essentially ineffective for approximating $\nabla f$ and we are left with one choice, $\psi$ with one hidden layer. From this perspective, the products of neural networks that emerged from the analysis of $\nabla \phi$ must have potentials for the problem of universal function approximation.
  
 \item[(iv)] What we proved here for the problem of approximating $\nabla f$ with deep neural networks could be a general phenomenon for approximating any other constrained mappings $g:\mathscr{S} \rightarrow \mathscr{S}'$, where perhaps the ``only'' option is to formulate the problem with a ``dual'' neural network, akin to $\phi$. In somewhat general terms, this can be framed as the problem of universal approximation of constrained mappings with implicit parametrization.
 \end{itemize}

{\bf Acknowledgements.} I would like to acknowledge nearly three years of discussions with Aapo Hyv\"arinen. I am also grateful to Francis Bach for discussions and valuable comments on the manuscript.

\bibliography{df.bib}

\newcommand{\etalchar}[1]{$^{#1}$}
\begin{thebibliography}{LPW{\etalchar{+}}17}

\bibitem[AB14]{alain2014regularized}
Guillaume Alain and Yoshua Bengio.
\newblock What regularized auto-encoders learn from the data-generating
  distribution.
\newblock {\em Journal of Machine Learning Research}, 15(1):3563--3593, 2014.

\bibitem[Ben09]{bengio2009learning}
Yoshua Bengio.
\newblock Learning deep architectures for {AI}.
\newblock {\em Foundations and Trends{\textregistered} in Machine Learning},
  2(1):1--127, 2009.

\bibitem[BPRS18]{baydin2018automatic}
Atilim~Gunes Baydin, Barak~A Pearlmutter, Alexey~Andreyevich Radul, and
  Jeffrey~Mark Siskind.
\newblock Automatic differentiation in machine learning: a survey.
\newblock {\em Journal of Machine Learning Research}, 18(153), 2018.

\bibitem[Gri03]{griewank2003mathematical}
Andreas Griewank.
\newblock A mathematical view of automatic differentiation.
\newblock {\em Acta Numerica}, 12:321--398, 2003.

\bibitem[Gri12]{griewank2012invented}
Andreas Griewank.
\newblock Who invented the reverse mode of differentiation?
\newblock {\em Documenta Mathematica, Extra Volume ISMP}, pages 389--400, 2012.

\bibitem[Hin99]{hinton1999products}
Geoffrey~E Hinton.
\newblock Products of experts.
\newblock In {\em International Conference on Artificial Neural Networks},
  pages 1--6, 1999.

\bibitem[Hyv05]{hyvarinen2005estimation}
Aapo Hyv{\"a}rinen.
\newblock Estimation of non-normalized statistical models by score matching.
\newblock {\em Journal of Machine Learning Research}, 6(Apr):695--709, 2005.

\bibitem[LCS19]{li2019annealed}
Zengyi Li, Yubei Chen, and Friedrich~T Sommer.
\newblock Annealed denoising score matching: Learning energy-based models in
  high-dimensional spaces.
\newblock {\em arXiv preprint arXiv:1910.07762}, 2019.

\bibitem[LPW{\etalchar{+}}17]{lu2017expressive}
Zhou Lu, Hongming Pu, Feicheng Wang, Zhiqiang Hu, and Liwei Wang.
\newblock The expressive power of neural networks: A view from the width.
\newblock In {\em Advances in Neural Information Processing Systems}, pages
  6231--6239, 2017.

\bibitem[RHW86]{rumelhart1986learning}
David~E Rumelhart, Geoffrey~E Hinton, and Ronald~J Williams.
\newblock Learning representations by back-propagating errors.
\newblock {\em Nature}, 323(6088):533, 1986.

\bibitem[Sch15]{schmidhuber2015deep}
J{\"u}rgen Schmidhuber.
\newblock Deep learning in neural networks: An overview.
\newblock {\em Neural Networks}, 61:85--117, 2015.

\bibitem[SH19]{saremi2019neural}
Saeed Saremi and Aapo Hyv{\"a}rinen.
\newblock Neural empirical {B}ayes.
\newblock {\em arXiv preprint arXiv:1903.02334}, 2019.

\bibitem[SMSH18]{saremi2018deep}
Saeed Saremi, Arash Mehrjou, Bernhard Sch{\"o}lkopf, and Aapo Hyv{\"a}rinen.
\newblock Deep energy estimator networks.
\newblock {\em arXiv preprint arXiv:1805.08306}, 2018.

\bibitem[Tao10]{tao2010epsilon}
Terence Tao.
\newblock {\em An epsilon of room, I: real analysis}.
\newblock American Mathematical Society, 2010.

\bibitem[Vin11]{vincent2011connection}
Pascal Vincent.
\newblock A connection between score matching and denoising autoencoders.
\newblock {\em Neural Computation}, 23(7):1661--1674, 2011.

\end{thebibliography}
\bibliographystyle{alpha}

\end{document}